\documentclass[lettersize,journal]{IEEEtran}
\usepackage{amsmath,amsfonts}
\usepackage[ruled,vlined,linesnumbered]{algorithm2e}
\usepackage[dvipsnames]{xcolor}
\usepackage{algorithmic}
\usepackage{array}
\usepackage[caption=false,font=normalsize,labelfont=sf,textfont=sf]{subfig}
\usepackage{textcomp}
\usepackage{stfloats}
\usepackage{url}
\usepackage{verbatim}
\usepackage{graphicx}
\usepackage{cite}
\usepackage{amsthm}
\usepackage{multirow}
\newtheorem{theorem}{Theorem}

\begin{document}

\title{IADD-TR: Intervention-Aware Dynamics Decoupling with Targeted Regularization for Model-Based Reinforcement Learning}

\author{Zefeng~Liang\textsuperscript{\textsection},
        Jie~Qiao\textsuperscript{\textsection},
        Ruichu~Cai,~\IEEEmembership{Senoir Member,~IEEE,}
        Weilin~Chen,
        and~Zhifeng~Hao,~\IEEEmembership{Member,~IEEE,}
\thanks{Zefeng Liang, Jie Qiao, and Weilin Chen are with  School of Computer Science, Guangdong University of Technology, Guangzhou, China. Email: lzfeng011021@gmail.com, qiaojie.chn@gmail.com, chenweilin.chn@gmail.com;}
\thanks{Ruichu Cai is with School of Computer Science, Guangdong University of Technology, Guangzhou, China. Email: cairuichu@gmail.com;}
\thanks{Zhifen Hao is with College of Science, Shantou University, Shantou, China. Email: zfhao@gdut.edu.cn}
}

\IEEEaftertitletext{%
\begin{center}
\begin{minipage}{0.92\textwidth}
\centering
\footnotesize\itshape
This work has been submitted to the IEEE for possible publication.
Copyright may be transferred without notice, after which this version
may no longer be accessible.
\end{minipage}
\end{center}
}
\markboth{Journal of \LaTeX\ Class Files,~Vol.~14, No.~8, August~2021}%
{Shell \MakeLowercase{\textit{et al.}}: A Sample Article Using IEEEtran.cls for IEEE Journals}


\maketitle

\begin{abstract}
Model-based reinforcement learning (MBRL), which learns environment dynamics to generate synthetic experience, is a promising approach to sample-efficient decision making. Numerous methods have been developed to improve dynamics prediction and policy optimization for MBRL through uncertainty estimation, model regularization, and conservative value learning. However, these methods typically treat the transition model and critic as monolithic predictors, overlooking the policy-induced data bias. Consequently, action can become entangled with environmental evolution, while uneven action coverage may distort the counterfactual value estimates used for policy improvement. To address this, we propose IADD-TR, a unified framework combining Intervention-Aware Dynamics Decoupling (IADD) and Targeted Regularization (TR).
IADD factorizes transitions into an action-intervention stage and an action-free natural evolution stage, using a zero-action anchor to resolve the non-uniqueness of this two-stage factorization for robust generalization. Its latent and state-aligned components are identifiable up to an invertible within-block transformation and pointwise, respectively.
For policy learning, we derive TR from the efficient influence function of a replay-state policy-gradient functional. TR augments the critic with an action-density-scaled residual correction and optimizes a targeted loss, yielding doubly robust policy-gradient estimation when either the critic or the replay action density is consistently specified. Extensive experiments on five MuJoCo tasks show that IADD-TR achieves competitive returns with improved sample efficiency. 
\end{abstract}


\section{Introduction}
\label{sec:introduction}
The pursuit of sample-efficient and generalizable decision-making agents has long been a central challenge in reinforcement learning (RL) \cite{sutton1998introduction}. Model-based reinforcement learning (MBRL) addresses this challenge by learning environment dynamics to predict the consequences of actions and then using the learned model to generate simulated experience without repeatedly querying the original environment. This capability is particularly valuable when real-world interactions are costly or difficult to obtain, such as in high-fidelity simulators, physical robotics, and human-involved decision-making environments. These advantages have fueled substantial interest in MBRL \cite{sutton1990integrated, luo2024survey, deisenroth2011pilco, levine2013guided, Hafner2020Dream, moerland2023model, schrittwieser2020mastering} and have led to advanced methods such as Model-based Policy Optimization (MBPO) \cite{janner2019trust}, which builds on Dyna-style learning \cite{sutton1991dyna} with a learned dynamics model and a Soft Actor-Critic (SAC) agent \cite{haarnoja2018soft}.

\begin{figure}
    \centering
    \includegraphics[width=0.48\textwidth]{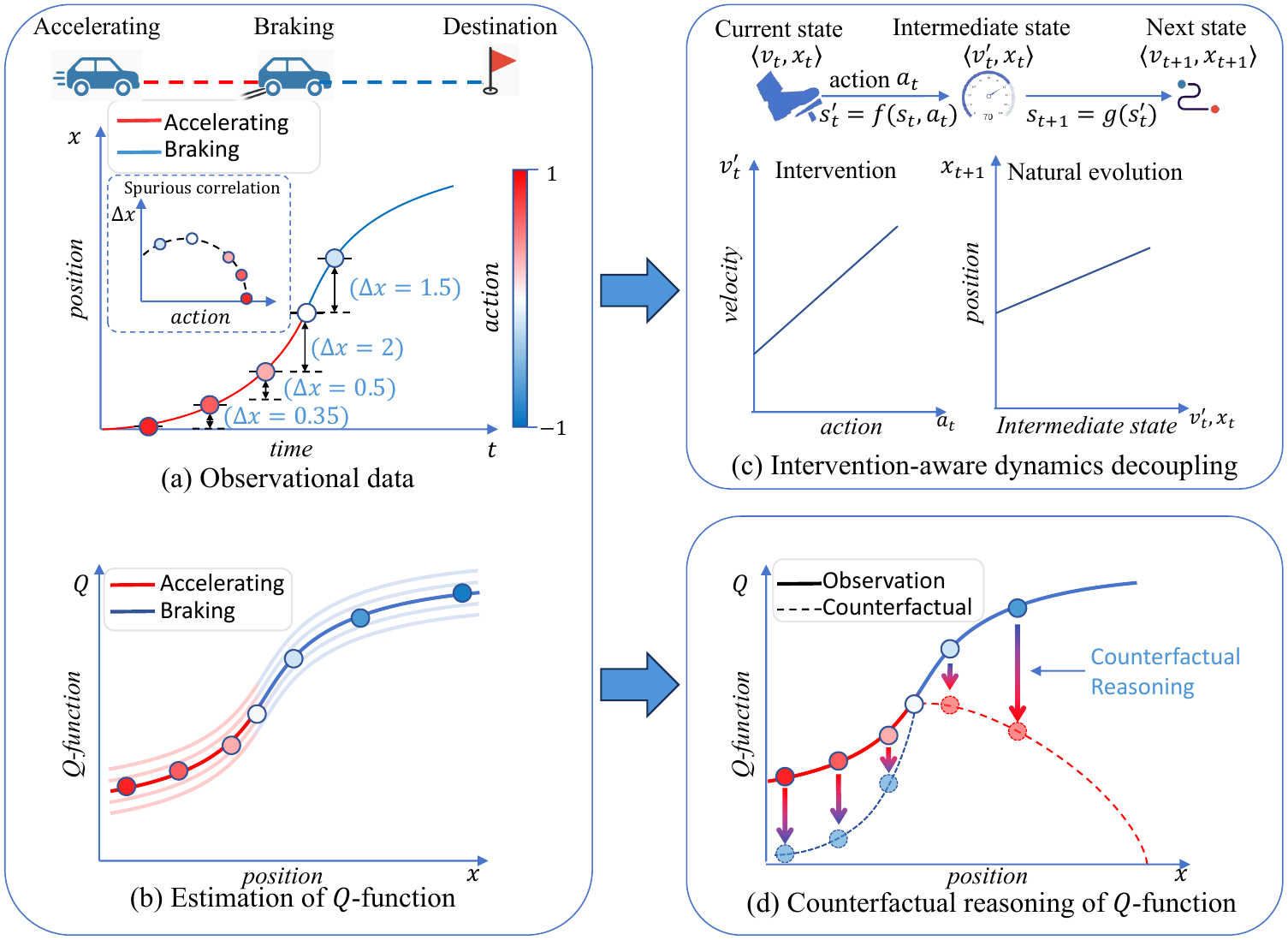}
    \caption{An example of confounding bias in a car-driving scenario, together with an illustration of intervention-aware dynamics decoupling and counterfactual reasoning for the $Q$-function.}
    \label{fig:introduction}
\end{figure}

Despite these successes, two interrelated challenges remain central to MBRL: (1) learning an accurate dynamics model from limited agent-collected data, and (2) learning an effective policy from imperfect model-generated data. At its core, both dynamics model learning and policy optimization rely on estimating complex mappings from data collected by the agent. For model learning, MBRL seeks to learn a one-step transition model $p(s_{t+1} | s_t, a_t)$; for policy learning, actor-critic methods often learn a value function such as $Q_{\omega}(s_t, a_t)$. These formulations, which turn complex dynamical and decision-making processes into direct function approximation problems, provide a convenient simplification and work well in many settings. However, they can become error-prone when interactions among states, actions, and subsequent transitions cause policy-generated data to contain \textit{confounding bias}. Consider the example in Fig.~\ref{fig:introduction}(a), in which a self-driving car is approaching a destination. Along the observed trajectories, the policy accelerates primarily near the starting region and brakes as the car approaches the destination. As a result, these policy-generated data would create spurious correlations such that a larger acceleration is correlated with smaller subsequent positions and a smaller acceleration is correlated with larger subsequent positions. While the model might perform well in this specific scenario, it would fail to generalize. This is because the position acts as a confounder, influencing both the action and the resulting state.

This confounding bias corrupts the learning process in two critical ways. First, for model learning, the model $p$ may learn spurious correlations from limited agent-collected data. For instance, it may wrongly associate deceleration with large positional gains, leading to poor generalization in new scenarios. Second, for policy learning, Q-value estimation can suffer from selection bias because the critic is trained only on the actions observed in the agent's experience. As illustrated in Fig.~\ref{fig:introduction}(b) and (d), this state-dependent action coverage makes it difficult to evaluate counterfactual actions, such as the outcome of accelerating at a larger position, thereby hindering the agent's ability to find a truly optimal policy.

Various approaches have been proposed to mitigate these challenges in model and policy learning. For model learning, techniques such as model ensembles \cite{kurutach2018modelensemble}, Lipschitz constraints \cite{venkatraman2015improving,asadi2018lipschitz}, and distribution matching methods \cite{chen2023adversarial,ho2016generative} have been used to improve the robustness of one-step transition estimation. However, these approaches still treat dynamics learning as a monolithic, black-box mapping problem. Similarly, for policy learning, methods from offline reinforcement learning address selection bias by either constraining the policy to the data distribution \cite{fujimoto2019off} or learning conservative Q-functions \cite{kumar2020conservative}. While effective at preventing catastrophic failures, these approaches do not directly address the confounding mechanism that induces biased dynamics and value estimates.

In this work, we focus on confounding bias arising from limited interactive data within the actor-critic framework for model-based reinforcement learning \cite{frauenknecht2024trust,janner2019trust}. The preceding analysis suggests that confounding bias is not merely a defect of the collected data but is also amplified by algorithmic choices inherited from monolithic black-box modeling. Such choices can lead to model misspecification and biased estimates from biased data \cite{Hernan2024-HERCIW}. Instead of treating the dynamics model and the critic as generic functions to be approximated, we deconstruct both learning targets to explicitly account for and mitigate confounding bias.

First, for model learning, we begin with the insight from causal inference that an agent's action is fundamentally a causal intervention in many environments. For instance, as shown in Fig. \ref{fig:introduction}(c), the action of acceleration is an intervention on the velocity over a short time interval. This perspective motivates our Intervention-Aware Dynamics Decoupling method, which reconstructs the transition process by aligning it with the underlying causal process present in many environments: (1) the action's immediate intervention, and (2) the environment's subsequent natural evolution. To model this, we use an intermediate state to capture the direct consequences of the intervention. The environment then transitions naturally from this intermediate state. Crucially, this two-stage process can be anchored in many environments by a reference zero action, namely an action that applies no direct intervention and leaves the intermediate observable state unchanged. By leveraging this zero-action condition, we can learn the natural evolution component and then separate the effect of nonzero interventions from the subsequent environment dynamics. This causal decomposition helps prevent the dynamics model from explaining action-induced changes as natural evolution, thereby reducing confounding bias and improving sample efficiency.

Second, for policy learning in a general actor-critic framework for MBRL, the Q-function is commonly treated as a black-box approximation learned via Bellman error minimization. Learning this approximation is difficult from confounded data, and more importantly, it is not the ultimate target of policy optimization. In its role as a critic, the Q-function should provide the information needed for an accurate and less biased estimate of the policy gradient. The standard Bellman-error objective can therefore be mismatched with the final goal of improving the actor through its gradient. Motivated by targeted learning in causal inference \cite{van2011targeted}, we introduce Targeted Regularization (TR) to align critic learning more directly with the policy-gradient target. Rather than only minimizing the prediction error of the Q-function, our approach regularizes the critic so that it produces a debiased and efficient estimate of the target parameter, guiding the policy update in a more reliable direction.

This work makes four key contributions toward addressing confounding bias in MBRL:
\begin{enumerate}
    \item We propose an intervention-aware dynamics decoupling model that separates action-induced intervention effects from the natural evolution of the environment.
    \item We develop a zero-action-based identification strategy for the two-stage decoupling, which anchors the natural-evolution component and enables the effect of nonzero interventions to be separated from subsequent environment dynamics.
    \item We introduce targeted regularization for Q-function learning to reduce the impact of confounding bias on policy-gradient estimation.
    \item We provide theoretical and empirical validation demonstrating the effectiveness of the proposed method.
\end{enumerate}



\section{Related Work}\label{sec:related_work}
Our work is related to several established fields, including model-based reinforcement learning (MBRL), offline reinforcement learning, policy evaluation, and causal inference.
The core challenge in MBRL is learning an accurate dynamics model and an effective policy from limited, agent-collected data \cite{moerland2023model}. Early work focused on local linear models \cite{levine2013guided,gu2016continuous} or Gaussian processes \cite{deisenroth2011pilco}, but in more complex environments, higher capacity models are required, such as neural networks \cite{nagabandi2018neural}. To further combat model error from limited data, a primary strategy is to use probabilistic models or ensemble methods to control uncertainty \cite{gal2016improving,buckman2018sample,chua2018deep, kurutach2018modelensemble}. This uncertainty is then used to regularize policy updates, preventing the agent from exploiting areas where the model is inaccurate, as seen in Dyna-style algorithms like MBPO \cite{janner2019trust}. However, these approaches still treat the dynamics as a monolithic, black-box mapping, leaving them vulnerable to learning spurious correlations when the collected data contains confounding bias.


The confounding bias issue is also closely related to the challenge of out-of-distribution (OOD) actions in offline reinforcement learning (offline RL) \cite{levine2020offline}. In offline RL, agents learn from a fixed dataset, facing the OOD problem when evaluating OOD actions of the Q-function \cite{fujimoto2019off}. To mitigate this, common methods either constrain the learned policy to remain close to the data-generating behavior policy \cite{fujimoto2019off} or learn a conservative Q-function that explicitly penalizes OOD actions \cite{kumar2020conservative,an2021uncertainty,cheng2022adversarially}. Model-based offline methods, such as MOPO \cite{yu2020mopo} and MOREL \cite{kidambi2020morel}, also adopt a conservative stance, using model uncertainty to construct a pessimistic MDP that penalizes policy learning in risky regions. While effective at preventing catastrophic failures, these approaches are inherently conservative. Our approach, in contrast, aims to debias the learning process itself.

For policy learning, our work addresses the confounding bias that corrupts the critic's estimates. This concern with the reliability of value estimates from biased data is also related to off-policy evaluation (OPE), which aims to accurately estimate a policy's value under distributional shift \cite{uehara2025review}. The OPE field has long bridged RL and causal inference by treating the evaluation of a policy as a counterfactual quantity estimation problem \cite{murphy2001, Dudik2011}. Typical methods include Direct Methods (DM) (analogous to G-computation \cite{robins1999estimation}), which have low variance but suffer from high bias if the model is misspecified, and Importance Sampling (IS) \cite{Precup2000} (analogous to inverse propensity scoring \cite{rosenbaum1983central}), which is unbiased but suffers from high variance \cite{farajtabar2018more}. This trade-off led to Doubly Robust (DR) estimators \cite{robins1994estimation, jiang2016doubly, thomas2016data} and the targeted learning estimator \cite{bibaut2019more}, combining DM and IS to reduce both bias and variance.

While sharing this causal perspective, we re-evaluate the goal of the critic. Unlike OPE, which focuses on estimating policy value optimally, we argue that in the actor-critic framework, the critic should guide the actor by estimating a value function (such as the Q-function) precisely for estimating an unbiased and efficient policy gradient, which is the primary optimization target. Inspired by Targeted Learning and Targeted Maximum Likelihood Estimation (TMLE) \cite{van2006targeted, van2011targeted}, we introduce targeted regularization. This approach optimizes the critic to be an efficient estimator for the policy gradient via the efficient influence function (EIF) of the target parameter, thereby correcting the bias induced by confounding.

\section{Background}

In this section, we introduce the basic notation for model-based reinforcement learning, off-policy replay data, and the policy-gradient target used throughout the paper.

We consider a Markov decision process (MDP) represented by the tuple $(\mathcal{S}, \mathcal{A}, p, r, \gamma, \rho_0)$. Here, $\mathcal{S}$ denotes the state space, and $\mathcal{A}$ denotes the action space; for vector-valued states, we write $s_t\in\mathbb{R}^{d_s}$, where $d_s$ denotes the state dimension. The transition kernel $p(s'\mid s,a)$ describes the environment dynamics, defining the transition probability to state $s'$ given state $s$ and action $a$; $r(s,a)$ is the reward function; $\gamma \in (0,1)$ is the discount factor; and $\rho_0$ is the initial state distribution. The goal of reinforcement learning is to learn a policy $\pi(a\mid s)$ that maximizes the expected cumulative discounted reward:
\begin{equation}
	\eta[\pi] = \mathbb{E}_{\tau \sim \pi}\left[\sum_{t=0}^\infty \gamma^t r(s_t, a_t)\right],
\end{equation}
where the trajectory $\tau = (s_0, a_0, s_1, \dots)$ is generated by rolling out the policy in the environment, with $s_0 \sim \rho_0$ and $a_t \sim \pi(\cdot\mid s_t)$. The Q-function represents $\eta[\pi]$ conditioned on a specific state-action pair and is given by
\begin{equation}
	Q^{\pi_\theta}(s_t, a_t) = \mathbb{E}_{\tau\sim\pi_\theta}\left[\sum_{k=t}^{\infty}\gamma^{k-t}r(s_k, a_k)\middle|s_t, a_t\right].
\end{equation}

In off-policy actor-critic learning, the replay buffer contains transitions collected by historical behavior policies. We denote the replay distribution by $\mathcal{D}$ and its state marginal by $\mathcal{D}_S$. For a replay action $a$ generated by its associated behavior policy $\pi_\beta(\cdot\mid s)$, we log the density
\begin{equation}
\label{eq:propensity_score}
e_{\pi_\beta}=\pi_\beta(a\mid s).
\end{equation}
Accordingly, a replay transition is written as $(s,a,e_{\pi_\beta},r,s')\sim\mathcal{D}$. When computing its Bellman target, we sample $a'\sim\pi_\theta(\cdot\mid s')$ from the current policy and evaluate the density of that sampled action as
\begin{equation}
e_{\pi_\theta}=\pi_\theta(a'\mid s').
\end{equation}
Thus, $e_{\pi_\beta}$ is the logged density of the replay action under its behavior policy, whereas $e_{\pi_\theta}$ is the density of the action under the current policy.

To optimize the policy using the learned model, in this work, we focus on the Dyna-style MBRL \cite{sutton1991dyna} with a soft actor-critic agent \cite{haarnoja2018soft}, which is widely adopted in state-of-the-art methods such as MBPO \cite{janner2019trust}. One of our central objects of interest is the policy-gradient target, which serves as the foundation for policy optimization via the policy gradient theorem \cite{sutton1999policy}. We denote this target by $\psi_{\mathrm{pg}}$, which is expressed as
\begin{equation}
\label{eq:policy_gradient}
	\psi_{\mathrm{pg}}
	= \nabla_\theta \eta[\pi_\theta]
	= \mathbb{E}_{\substack{s \sim d^{\pi_\theta} \\ a \sim \pi_\theta(\cdot\mid s)}} \left[ \nabla_\theta \log \pi_\theta(a\mid s) \cdot Q^{\pi_\theta}(s,a) \right],
\end{equation}
where $d^{\pi_\theta}(s)$ is the discounted state visitation distribution under $\pi_\theta$, and $Q^{\pi_\theta}(s,a)$ is the action-value function. In this paradigm, the actor and critic are typically implemented as neural networks for estimating the policy $\pi_\theta(a\mid s)$ with parameter $\theta$ and the value function $Q_\omega(s,a)$ with parameter $\omega$, respectively.

\section{Methodology}
\begin{figure*}
\centering
\includegraphics[width=1.0\textwidth]{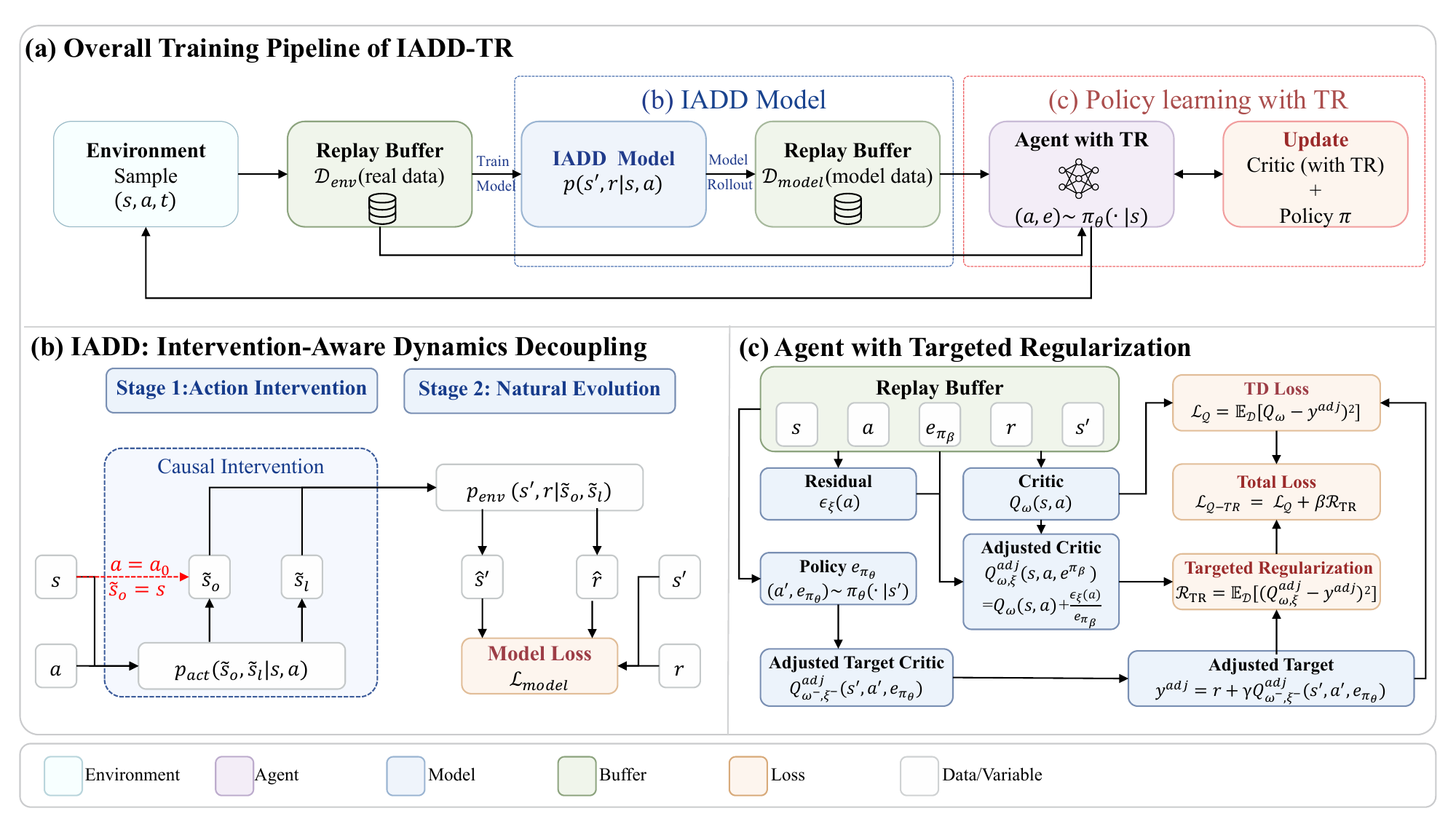}
\caption{Illustration of the proposed IADD-TR method.
The framework consists of two key components: IADD and targeted regularization (TR).
(a) In the overall training pipeline, the agent interacts with the environment to collect real transitions, which are used to train the dynamics model and to update the policy.
The learned model further generates synthetic transitions for model-based policy optimization.
(b) IADD adopts a two-stage dynamics model, which decomposes the transition prediction process into an intermediate-stage prediction and a final reward/next-state prediction.
(c) TR modifies the critic learning objective by introducing an adjusted critic and a targeted regularization term, improving the robustness of value estimation under imperfect model rollouts.}
\label{fig:framework}
\end{figure*}
In this section, we describe the Intervention-Aware Dynamics Decoupling with Targeted Regularization (IADD-TR) framework in the context of model-based reinforcement learning.
As illustrated in Fig.~\ref{fig:framework}(a), IADD-TR follows an iterative model-based training pipeline.
The agent first interacts with the environment to collect real transitions, which are stored in the environment replay buffer and used for both dynamics model learning and policy optimization.
The learned IADD dynamics model then generates synthetic transitions, which are stored in the model replay buffer and further used to augment policy learning.
During policy learning, targeted regularization is incorporated into the critic update to mitigate selection bias in value estimation.
Overall, the proposed framework is designed to address two sources of error that arise during training: model learning bias and policy learning bias.

\subsection{Addressing Model Learning Bias}

A key challenge in dynamics model learning is that action-induced changes and environment-induced changes are often entangled in the collected transition data, which can introduce confounding bias into transition prediction. To address this issue, we do not model the transition as a monolithic mapping from \((s_t,a_t)\) to \(s_{t+1}\). Instead, as illustrated in Fig.~\ref{fig:framework}(b), we explicitly decompose each one-step transition into an action-intervention stage and a natural natural evolution stage. This intervention-aware decomposition makes the action effect explicit, thereby reducing its confounding with natural evolution during dynamics learning.

Formally, we introduce an intermediate representation $\tilde{s}_t=(\tilde{s}_{t,o},\tilde{s}_{t,l})$ to characterize the post-intervention outcome induced by the action. The variable $\tilde{s}_{t,o}\in\mathbb{R}^{d_s}$ denotes the observable post-intervention state. It is designed to have the same dimensionality and semantic structure as the original state $s_t$, i.e., $d_o=d_s$, and is designed to capture the direct causal effect of the action. Moreover, to retain additional action-dependent information needed by the subsequent evolution model, we introduce an auxiliary latent component $\tilde{s}_{t,l}\in\mathbb{R}^{d_l}$. Together, $\tilde{s}_{t,o}$ and $\tilde{s}_{t,l}$ form the post-intervention representation used for predicting the next state. We write the action-intervention stage as the deterministic structural mechanism
\begin{equation}
\label{eq:deterministic_action_intervention}
\tilde{s}_{t,o}
=
p_{\mathrm{act},o}(s_t,a_t),
\qquad
\tilde{s}_{t,l}
=
p_{\mathrm{act},l}(s_t,a_t,\varepsilon_t),
\end{equation}
where $p_{\mathrm{act},o}$ and $p_{\mathrm{act},l}$ are the observable and latent components of $p_{\mathrm{act}}$, respectively, and $\varepsilon_t$ denotes exogenous noise. Thus, the mechanism is deterministic given $(s_t,a_t,\varepsilon_t)$, while the randomness in $\varepsilon_t$ induces a conditional density for $\tilde{s}_{t,l}$. We assume that this latent conditional density factorizes as
\begin{equation}
\label{eq:act_factorization}
p_l(\tilde{s}_{t,l}\mid s_t,a_t)
=
\prod_{i=1}^{d_l} p_{l,i}(\tilde{s}_{t,l,i}\mid s_t,a_t),
\end{equation}
where \(\tilde{s}_{t,l}=(\tilde{s}_{t,l,1},\dots,\tilde{s}_{t,l,d_l})\). The natural evolution stage is represented by the deterministic mechanism $p_{\mathrm{env}}$, and the resulting one-step transition model is the composition
\begin{equation}
\label{eq:induced_transition_model}
s_{t+1}
=
p_{\mathrm{env}}\!\left(
p_{\mathrm{act},o}(s_t,a_t),
p_{\mathrm{act},l}(s_t,a_t,\varepsilon_t)
\right).
\end{equation}

Note that the implementation also includes the reward-prediction head shown in Fig.~\ref{fig:framework}(b).

While this two-stage formulation separates the transition, the intermediate representation $(\tilde{s}_{t,o},\tilde{s}_{t,l})$ is strictly latent and not directly observed in the collected data. As a result, there is no guarantee that the learned intermediate representation corresponds to the intended causal post-intervention state, which can in turn compromise generalization performance.

To anchor the intermediate representation in physical reality and study its identifiability, we introduce a reference zero action $a_0$. Intuitively, a zero action represents the absence of any direct physical intervention. Under $a_0$, the observable state should remain unchanged immediately after the action-intervention stage. Formally, we define the zero-action-anchored IADD model class by the condition $p_{\mathrm{act},o}(s_t,a_0)=s_t$. This zero-action condition aligns the observable post-intervention coordinates with the original state space on the zero-action support. Where zero-action data cover the relevant reachable support, the same alignment extends to intermediate representations reached under nonzero actions.

This equality is enforced by construction rather than through an auxiliary loss. Whenever the input action equals $a_0$, the action-intervention stage directly copies $s_t$ into the observable coordinates of its intermediate output, so that $\tilde{s}_{t,o}=s_t$ exactly. The two stages are otherwise trained jointly using the standard one-step prediction objective
\begin{equation}
\label{eq:model_learning_objective}
\mathcal{L}_{\mathrm{model}}
=
\mathbb{E}_{(s_t,a_t,s_{t+1})\sim\mathcal{D}}
\left[
\left\|\hat{s}_{t+1}-s_{t+1}\right\|_2^2
\right].
\end{equation}
Thus, the exact zero-action equality is part of the IADD model-class definition and introduces neither an additional identifiability assumption nor a weighting hyperparameter in the learning objective.

With this anchor, we investigate the identifiability of the post-intervention representations by examining the uniqueness of the learned variables. To formalize this, let $(p_{\mathrm{act}}, p_{\mathrm{env}}, \tilde{s}_{t,o}, \tilde{s}_{t,l})$ denote the true parameterization of the mechanisms and intermediate variables. We then consider an observationally equivalent alternative, denoted with bars $(\bar{p}_{\mathrm{act}}, \bar{p}_{\mathrm{env}}, \bar{s}_{t,o}, \bar{s}_{t,l})$. Identifiability asks how strictly the equality of the observed transition distributions constrains the relationship between these two intermediate representations under the same marginal distribution. Since the observable component is grounded by the state of zero-action, we expect an exact pointwise recovery, such that $\bar{s}_{t,o} = \tilde{s}_{t,o}$. Conversely, for the unobserved latent component, we seek identifiability up to a permutation and component-wise invertible transformations. 

The following theorem formalizes this pointwise recovery guarantee in the population limit.
\begin{theorem}
\label{thm:point_ident}
\textbf{(Pointwise Identifiability of $\tilde{s}_{t,o}$.)}
Consider two parameterizations in the zero-action-anchored IADD model class that induce the same conditional one-step transition distribution in Eq.~\eqref{eq:induced_transition_model} for every reachable state-action context. Suppose the following condition holds:
\begin{itemize}
    \item A1 (\underline{Injectivity and support inverse}): The deterministic natural evolution mechanisms $p_{\mathrm{env}}$ and $\bar p_{\mathrm{env}}$ are injective on their reachable intermediate supports, admit continuous inverses on their common reachable next-state support, and are twice continuously differentiable on these supports.
\end{itemize}
Then the observable post-intervention state is pointwise identifiable for every reachable $(s_t,a_t)$ whose conditional next-state support is covered by zero-action contexts:
\[
\bar{s}_{t,o}(s_t,a_t)
=
\tilde{s}_{t,o}(s_t,a_t)
.\]
\end{theorem}

The theorem shows that the hard zero-action anchor makes the observable post-intervention state pointwise identifiable wherever zero-action data cover the relevant reachable support. In population terms, this requires sufficiently broad zero-action support. With finite data, it motivates collecting a sufficiently large and diverse set of zero-action transitions to approximate that coverage. Intuitively, zero-action transitions fix the observable coordinates, and the invertibility in A1 extends this calibration to nonzero-action states within the covered support.

We next consider the latent component. Once Theorem~\ref{thm:point_ident} has fixed the observable component pointwise, the remaining task is to identify the latent coordinates. The following theorem uses the factorized action-intervention distribution and its variation across state-action contexts to establish component-wise identifiability.

\begin{theorem}
\label{thm:latent_component_ident}
\textbf{(Component-wise Identifiability of $\tilde{s}_{t,l}$.)}
On the zero-action-covered reachable support where Theorem~\ref{thm:point_ident} establishes pointwise identifiability, suppose the factorized action-intervention model in Eq.~\eqref{eq:act_factorization} holds. For every fixed reachable value $s'_o$ of the observable post-intervention state, suppose further that:
\begin{itemize}
    \item B1 (\underline{Regularity}): The conditional latent component densities are positive and twice continuously differentiable on connected supports that locally contain product neighborhoods.
    \item B2 (\underline{Sufficient variability}): State-action contexts producing that observable state induce sufficiently diverse variations in the conditional latent distributions.
\end{itemize}
The precise full-rank variability condition is given in Appendix~\ref{app:proof_latent}. Then there exist a permutation $\sigma_{l,s'_o}$ and one-dimensional diffeomorphisms $h_{1,s'_o},\dots,h_{d_l,s'_o}$ such that
\begin{equation}
\label{eq:conditional_latent_component_ident}
\bar{s}_{t,l,j}
=
h_{j,s'_o}\!\left(
\tilde{s}_{t,l,\sigma_{l,s'_o}(j)}
\right),
\qquad j=1,\dots,d_l.
\end{equation}
\end{theorem}
Once Theorem~\ref{thm:point_ident} has identified the observable component, component-wise identification of the latent component requires no additional zero-action information. It instead follows from B1 and B2, which are standard conditions in auxiliary-variable nonlinear ICA \cite{hyvarinen2019nonlinear}. B1 imposes smoothness and support regularity, while B2 requires enough variation across state-action contexts to distinguish the latent factors. Under these conditions, the latent coordinates are identifiable up to permutation and one-dimensional invertible transformations.

Together, the two theorems identify the observable component exactly and the latent component componentwise. The proofs are given in Appendices~\ref{app:proof_pointwise} and~\ref{app:proof_latent}.

\subsection{Addressing Policy Learning Bias}

Beyond model-learning bias, confounding can also induce policy-learning bias. To see this, in the actor-critic framework, the critic $Q_\omega(s,a)$ learns from a replay buffer $(s,a,e_{\pi_\beta},r,s')\sim\mathcal{D}$ collected by historical behavior policies, and is trained by minimizing the temporal-difference (TD) loss:
\begin{equation}
\label{eq:td_loss}
\mathcal{L}_{\mathrm{TD}}(\omega)
=
\mathbb{E}_{\mathcal{D}}\left[
\bigl(Q_\omega(s,a)-y\bigr)^2
\right],
\end{equation}
where
\begin{equation}
\label{eq:td_target_estimated}
y
=
r + \gamma\,\mathbb{E}_{a'\sim\pi_\theta(\cdot\mid s')}
\bigl[Q_{\omega^-}(s',a')\bigr]
\end{equation}
is the Bellman target. Because replay actions are sampled under the behavior density~$\pi_\beta(a\mid s)$, the collected distribution can be biased in the sense that some actions can be observed more frequently than others for the same state. As a result, the standard TD update forces the critic to fit high-density actions more accurately while leaving the value estimates for rare or counterfactual actions highly erroneous. This uneven accuracy is critical because the critic guides the actor through the policy gradient, which requires reliable estimates for counterfactual actions—the actions most affected by confounding bias. To formally investigate the nature of this bias, we evaluate the policy-gradient target directly rather than focusing solely on the Bellman loss. We formulate this target with respect to the replay state marginal $\mathcal{D}_S$ as follows:
\begin{equation}
\label{eq:replay_pg_target}
\begin{aligned}
\psi^{\mathcal{D}}_{\mathrm{pg}}
&=
\mathbb{E}_{\substack{s\sim\mathcal{D}_S\\
a\sim\pi_\theta(\cdot\mid s)}}
\left[
Q^{\pi_\theta}(s,a)\,
g_\theta(s,a)
\right],
\\
g_\theta(s,a)
&=
\nabla_\theta\log\pi_\theta(a\mid s).
\end{aligned}
\end{equation}
In practice, the actor update uses the learned critic $Q_\omega$ rather than the ideal action-value function $Q^{\pi_\theta}$, yielding the empirical plug-in estimate:
\begin{equation}
\label{eq:pg_plugin}
\widehat\psi^{\mathcal{D}}_{\mathrm{pg}}
=
\mathbb{E}_{\substack{s\sim\mathcal{D}_S\\ a\sim\pi_\theta(\cdot\mid s)}}
\left[
Q_\omega(s,a)\,g_\theta(s,a)
\right].
\end{equation}
The gap $\widehat\psi^{\mathcal{D}}_{\mathrm{pg}}-\psi^{\mathcal{D}}_{\mathrm{pg}}$ characterizes the policy learning bias. To rigorously trace this gap, we adopt a causal and semiparametric perspective, analyzing the first-order sensitivity of the target parameter $\psi^{\mathcal{D}}_{\mathrm{pg}}$ under the replay distribution $\mathcal{D}$. This sensitivity is mathematically represented by the efficient influence function (EIF). Let $\varphi^{\mathcal{D}}_{\mathrm{pg}}$ denote the EIF of $\psi^{\mathcal{D}}_{\mathrm{pg}}$. Under standard regularity conditions, the leading first-order behavior of the bias satisfies:
\begin{equation}
\label{eq:pg_asymptotic_expansion}
\sqrt{n}
\bigl\{
\widehat\psi^{\mathcal{D}}_{\mathrm{pg}}
-
\psi^{\mathcal{D}}_{\mathrm{pg}}
\bigr\}
=
\frac{1}{\sqrt{n}}
\sum_{i=1}^{n}
\varphi^{\mathcal{D}}_{\mathrm{pg}}(o_i)
+
o_P(1),
\end{equation}
where $o_i=(s_i,a_i,y_i^{\pi_\theta})$ represents replay samples augmented with the ideal one-step Bellman variable $y^{\pi_\theta} = r + \gamma\,Q^{\pi_\theta}(s',a')$. The explicit formulation of the EIF formally reveals how the data-generating mechanism dictates this bias, as detailed in the following theorem.
\begin{theorem}
\label{thm:eif_pg}
Consider the replay-state policy-gradient functional $\psi^{\mathcal{D}}_{\mathrm{pg}}$ in Eq.~\eqref{eq:replay_pg_target}, where the state marginal is $\mathcal{D}_S$ and the replay action density is $\pi_\beta(a\mid s)$. Treating $\pi_\theta$ and $g_\theta$ as fixed when taking the pathwise derivative, under the nonparametric replay model for $(s,a,y^{\pi_\theta})$, the efficient influence function is
\begin{equation}
\label{eq:eif_pg}
\begin{aligned}
\varphi^{\mathcal{D}}_{\mathrm{pg}}(o)
&=
\mathbb{E}_{\tilde a\sim\pi_\theta(\cdot\mid s)}\Bigg[
\Bigg(
\frac{\delta(a-\tilde a)}{\pi_\beta(\tilde a\mid s)}
\bigl(y^{\pi_\theta}-Q^{\pi_\theta}(s,\tilde a)\bigr)
\\
&\qquad\qquad\qquad
+
Q^{\pi_\theta}(s,\tilde a)
\Bigg)
 g_\theta(s,\tilde a)
\Bigg]
-
\psi^{\mathcal{D}}_{\mathrm{pg}},
\end{aligned}
\end{equation}
where $\delta(a-\tilde a)$ denotes the Dirac delta centered at the replay action $a$.
\end{theorem}

Theorem~\ref{thm:eif_pg} formally uncovers the source of the policy learning bias, where the inverse-probability term $1/\pi_\beta(\tilde a\mid s)$ in Eq.~\eqref{eq:eif_pg} amplifies the critic's residual error for actions rarely observed in the replay buffer. Since standard TD learning (Eq.~\eqref{eq:td_loss}) is blind to this gradient-specific sensitivity, it cannot eliminate the bias on its own.

To mitigate this issue, motivated by targeted learning in the causal inference literature \cite{van2011targeted}, we develop a targeted regularization framework. By constructing an adjusted critic that targets the EIF score equation, we can systematically neutralize the magnified errors on low-propensity actions, effectively resolving the policy learning bias at its source. As illustrated in Fig.~\ref{fig:framework}(c), we construct an adjusted critic that augments the baseline critic~$Q_\omega$ with a density-scaled residual term:
\begin{equation}
\label{eq:adjusted_critic}
Q^{\mathrm{adj}}_{\omega,\xi}(s,a,e_{\pi_\beta})
=
Q_\omega(s,a)
+
\frac{\epsilon_\xi(a)}{e_{\pi_\beta}},
\end{equation}
where $e_{\pi_\beta}$ is the paired action-selection score defined in Eq.~\eqref{eq:propensity_score}, and $\epsilon_\xi(a)$ is a learnable residual correction parameterized by~$\xi$. The same adjusted form is used for the target network:
\begin{equation}
\label{eq:adjusted_td_target}
y^{\mathrm{adj}}
=
r + \gamma\,
Q^{\mathrm{adj}}_{\omega^-,\xi^-}(s',a',e_{\pi_\theta}),
\end{equation}
where $a'$ is the target action and $e_{\pi_\theta}$ is its paired action-selection score from $\pi_\theta$.

We learn the residual correction~$\epsilon_\xi$, while keeping the baseline critic parameters~$\omega$ fixed, by minimizing the targeted regularization loss
\begin{equation}
\label{eq:tr_loss}
\mathcal{R}_{\mathrm{TR}}(\xi)
=
\mathbb{E}_{\mathcal{D}}\left[
\bigl(
Q^{\mathrm{adj}}_{\omega,\xi}(s,a,e_{\pi_\beta}) - y^{\mathrm{adj}}
\bigr)^2
\right].
\end{equation}
The overall critic objective combines the standard TD loss and the targeted regularization term:
\begin{equation}
\label{eq:total_loss}
\mathcal{L}_{\mathrm{IADD-TR}}(\omega,\xi)
=
\mathcal{L}_{\mathrm{TD}}(\omega)
+
\beta\,
\mathcal{R}_{\mathrm{TR}}(\xi),
\end{equation}
where $\beta>0$ controls the strength of the targeted regularization. We use the adjusted critic in the policy-gradient plug-in estimator
\begin{equation}
\label{eq:targeted_pg_estimator}
\widehat\psi^{\mathcal{D},\mathrm{TR}}_{\mathrm{pg}}
=
\frac{1}{n}\sum_{i=1}^{n}
\mathbb{E}_{\tilde a\sim\pi_\theta(\cdot\mid s_i)}
\left[
   Q^{\mathrm{adj}}_{\omega,\xi}(s_i,\tilde a)
 g_\theta(s_i,\tilde a)
\right].
\end{equation}
The following result summarizes the statistical role of this estimator; its empirical EIF construction, regularity conditions, and proof are deferred to Appendix~\ref{app:proof_tr}.

\begin{theorem}
\label{thm:tr_asymptotic_targeting}
\textbf{(Double robustness of targeted policy-gradient estimation.)}
Under the overlap, moment, and empirical-process conditions stated in Appendix~\ref{app:proof_tr}, and provided that the fluctuation update asymptotically solves the corresponding EIF score equation, the estimator in Eq.~\eqref{eq:targeted_pg_estimator} is consistent if either the targeted critic is consistent for $Q^{\pi_\theta}$ or the replay action-density estimate is consistent for the true replay density. In either regime,
\begin{equation}
\label{eq:tr_pg_rate}
\left\|
\widehat\psi^{\mathcal{D},\mathrm{TR}}_{\mathrm{pg}}
-
\psi^{\mathcal{D}}_{\mathrm{pg}}
\right\|
=
o_P(1).
\end{equation}
\end{theorem}

Here double robustness refers to the critic and replay action density; the fluctuation parameter is only a device for enforcing the targeted score equation. Because the behavior-policy density is logged with each replay action in our setting, the replay-density branch does not require fitting an additional propensity model. This property makes the actor update less sensitive to critic misspecification under uneven replay-action coverage.

\section{Experiments}
\label{sec:experiments}

In this section, we evaluate the proposed method on standard MuJoCo continuous-control benchmarks~\cite{todorov2012mujoco} and use two controlled diagnostics to test the mechanisms behind its model and policy-learning components.
The experiments are organized around four questions:
\textbf{Q1.} Does IADD-TR improve sample efficiency and final control performance relative to representative model-free and model-based baselines?
\textbf{Q2.} Do the two-stage dynamics model and targeted regularization contribute distinct and complementary performance gains?
\textbf{Q3.} Does the zero-action-anchored model recover the simulator-defined post-intervention state without sacrificing next-state prediction accuracy?
\textbf{Q4.} Does the targeted critic yield a policy-gradient estimate that is more directionally accurate relative to a finite-sample Monte Carlo (MC) reference?

\subsection{Experiment Settings}
\label{sec:env_baselines}

We conduct experiments on five widely used MuJoCo continuous-control benchmarks: HalfCheetah, Hopper, Walker2d, Ant, and Humanoid. All benchmark runs use MuJoCo 2.0.0.
These environments cover locomotion tasks with different levels of contact dynamics, control complexity, and long-horizon stability requirements.
To isolate the identifiability mechanism of the two-stage model, we additionally construct a controlled synthetic dynamics environment in which the simulator explicitly generates and records the observable post-intervention state.
The toy system has state $s_t=(x_t,v_t)$ and scalar action $a_t$; the action first changes velocity, producing $\tilde{s}_{t,o}$, and an action-independent damped-oscillator ODE then maps $\tilde{s}_{t,o}$ to $s_{t+1}$.
The construction satisfies the exact zero-action anchor $\tilde{s}_{t,o}=s_t$ when $a_t=0$.
The simulator-recorded $\tilde{s}_{t,o}$ is used only for evaluation, not for fitting the recovery model; the full environment definition and protocol are given in Appendix~\ref{app:synthetic_dynamics_study}.

We compare our method with both model-free and model-based reinforcement learning baselines.
For model-free reinforcement learning, we use \textbf{SAC}~\cite{haarnoja2018soft} and \textbf{PPO}~\cite{schulman2017proximal}.
SAC is an off-policy actor-critic method that optimizes an entropy-regularized objective, while PPO is a widely used on-policy policy-gradient method based on clipped policy updates.
For model-based reinforcement learning, we consider \textbf{SLBO}~\cite{luo2018algorithmic} and \textbf{MBPO}~\cite{janner2019trust}.
SLBO performs model rollouts from the initial-state distribution, whereas MBPO generates short-horizon model rollouts from states visited by previous policies.

Zero-action collection is enabled only for IADD-TR; SAC, PPO, SLBO, and MBPO retain their original action-sampling procedures. At each IADD-TR environment interaction, with probability $\lambda_0=0.1$, we replace the action proposed by SAC with the all-zero vector $a_0=\mathbf{0}$ before executing it; otherwise, the SAC action is retained. The resulting real zero-action transitions, including their observed rewards and next states, are used only to train a separate zero-action dynamics model. This model then generates zero-action transitions used exclusively for IADD dynamics-model training, where they support the anchor $\tilde{s}_{t,o}=s_t$. Neither the real nor model-generated zero-action transitions are added to the actor or critic replay buffers, including the replay data used for TR. Because action replacement occurs within the scheduled interactions, rather than through additional data collection, all methods use the same real-environment interaction budget. Thus, the zero-action mechanism neither alters the baseline replay buffers nor gives IADD-TR access to additional real transitions; it is evaluated as an internal component of the proposed method.

For IADD-TR, we fix the two additional settings introduced by the proposed components across all MuJoCo environments: the real zero-action collection ratio is set to $0.1$, and the targeted-regularization weight is set to $\beta=1$.
The remaining training protocol is kept aligned with MBPO, including the epoch length, model-rollout schedule, learning rates, and evaluation frequency; the detailed settings are reported in Appendix~\ref{app:implementation_details}.
For all MuJoCo methods, we report results across five independent runs with different random seeds.
We evaluate the policy once per epoch, corresponding to every 1,000 environment steps, and report the mean return with standard deviation.
The controlled synthetic dynamics study uses a fixed three-member ensemble to isolate the effect of the zero-action ratio; its complete protocol is reported in Appendix~\ref{app:synthetic_dynamics_study}.

\subsection{Results}
\label{sec:results_discussion}

\begin{figure*}[t]
\centering
\includegraphics[width=0.6\textwidth, keepaspectratio]{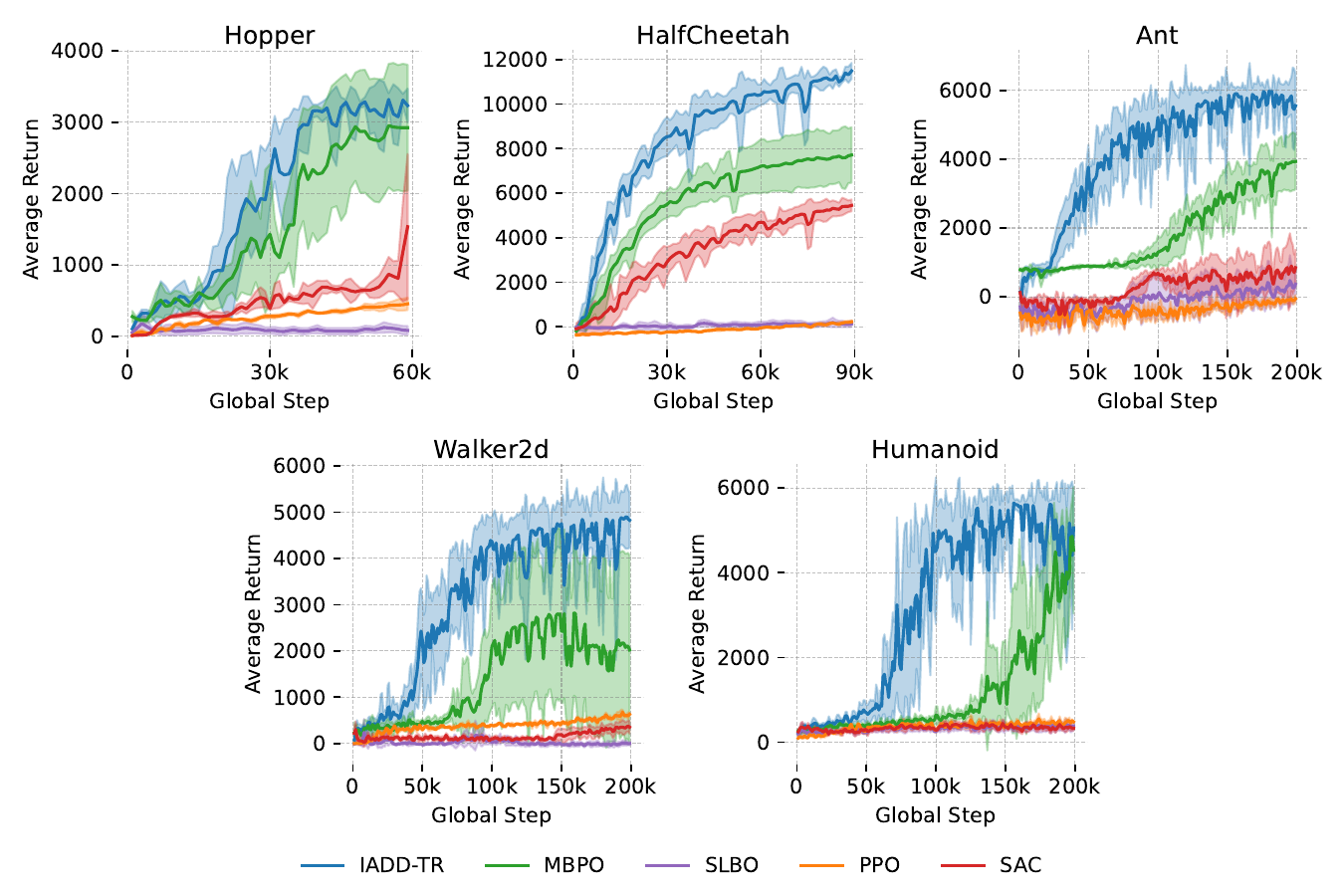}
\caption{Performance comparison on MuJoCo continuous-control benchmarks. We compare the proposed method (IADD-TR) with two model-free baselines (SAC and PPO) and two model-based baselines (SLBO and MBPO). Solid curves denote the mean return over five random seeds, and the shaded regions indicate one standard deviation.}
\label{fig:comparative_result}
\end{figure*}

\textbf{End-to-End Control Performance (Q1).}
Fig.~\ref{fig:comparative_result} tests whether the proposed model and critic designs translate into more effective policy learning under a fixed real-environment interaction budget. IADD-TR attains the highest or competitive returns across all five tasks while improving more rapidly than SAC, PPO, and SLBO. Relative to the strongest model-based comparator, MBPO, the gain is clearest on HalfCheetah, Walker2d, Ant, and Humanoid; on Hopper, IADD-TR improves faster early in training and remains competitive at convergence.

The advantage is particularly pronounced on Ant and Humanoid, where the higher-dimensional dynamics make policy learning more sensitive to errors in model-generated data and critic estimates. These results establish the end-to-end consequence of the proposed design: IADD-TR preserves the sample-efficiency benefit of short model rollouts while producing stronger policy performance on most tasks. Because aggregate return cannot identify which component produces the gain, Q2--Q4 separately examine component contributions and their intended mechanisms.

\begin{figure*}[t]
\centering
\includegraphics[width=0.6\textwidth, keepaspectratio]{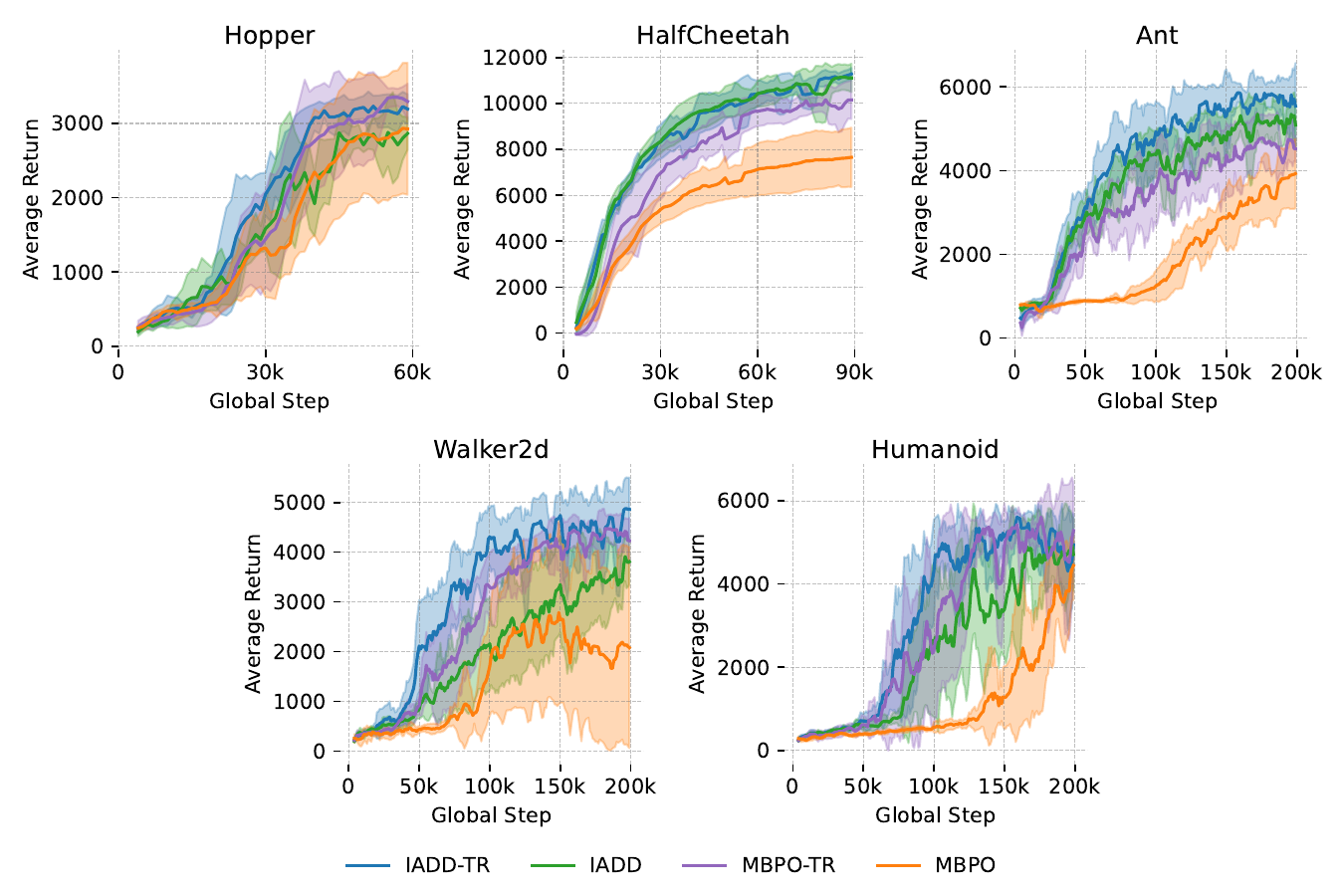}
\caption{Ablation results on MuJoCo continuous-control benchmarks. We compare MBPO, MBPO with targeted regularization (MBPO+TR), the proposed IADD, and the full method that combines both components (IADD+TR). Solid curves denote the mean return over five random seeds, and shaded regions indicate one standard deviation.}
\label{fig:ablation}
\end{figure*}

\textbf{Complementary Contributions of IADD and TR (Q2).}
The ablation in Fig.~\ref{fig:ablation} separates the two design claims: IADD modifies how model-generated transitions are constructed, whereas TR modifies how the critic is used to estimate the policy gradient. We compare MBPO, MBPO+TR, IADD, and the full IADD+TR method under the same real-environment interaction budget, training schedule, and random seeds. Zero-action collection is used only by the IADD variants, so each comparison preserves the sampling procedure associated with the corresponding model design.

Adding either IADD or TR generally improves over MBPO, showing that neither component is useful only in the presence of the other. Their effects are also task dependent. IADD provides the clearer individual gain on HalfCheetah, whereas MBPO+TR is particularly strong on Hopper. On Walker2d and Ant, combining the components produces the clearest improvement in learning speed and final return. On Humanoid, all three enhanced variants substantially improve the learning trajectory relative to MBPO, with IADD+TR maintaining strong performance throughout training.

The consistent performance of the combined method supports the intended division of labor: the model component improves the construction of synthetic experience, and TR improves the critic signal used by the actor. The return curves establish that the components are complementary, but they do not by themselves verify these internal mechanisms. Q3 therefore tests the model representation directly, while Q4 evaluates the policy-gradient estimate directly.

\begin{figure*}[t]
\centering
\includegraphics[width=0.95\textwidth, keepaspectratio]{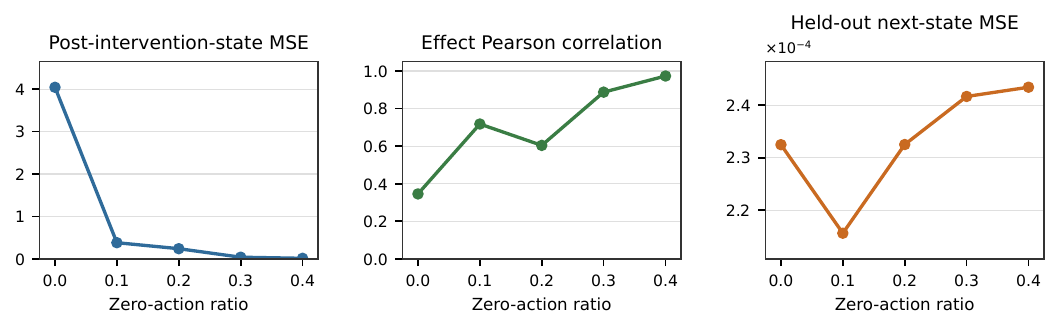}
\caption{Controlled synthetic dynamics study. The panels show the observable post-intervention-state MSE (left), action-effect Pearson correlation (middle), and held-out next-state MSE (right) over zero-action ratios from 0 to 0.4. Each curve reports the training run evaluated on 10,000 held-out nonzero-action transitions.}
\label{fig:synthetic_dynamics_study}
\end{figure*}

\textbf{Validation of the Zero-Action Model Design (Q3).}
Q3 tests the model-specific claim that the hard zero-action definition guides the learned decomposition toward the intended physical intermediate state, rather than merely fitting $s_{t+1}$ through an arbitrary hidden representation. We use the controlled synthetic dynamics environment detailed in Appendix~\ref{app:synthetic_dynamics_study}, in which the simulator records the post-intervention state $\tilde{s}_{t,o}$ before an action-independent natural evolution produces $s_{t+1}$. This recorded state is never used as a nonzero-action training target. The observable-state MSE and action-effect correlation therefore measure recovery of the intended intermediate coordinate. Held-out next-state MSE checks whether this recovery compromises the model's predictive role. We vary the zero-action ratio from 0 to 0.40 with 100,000 training transitions and 10,000 held-out nonzero-action transitions. Figure~\ref{fig:synthetic_dynamics_study} reports the completed sweep for random seed 1.

Increasing zero-action coverage substantially improved the two representation-sensitive metrics in the reported run. From $\lambda_0=0$ to $\lambda_0=0.40$, the post-intervention-state MSE decreased from $4.043$ to $0.0109$, a $99.7\%$ reduction. Over the same range, the action-effect correlation increased from $0.346$ to $0.973$. The MSE had already reached $0.379$ at $\lambda_0=0.10$ and then decreased at every higher ratio. Thus, the anchor changes what the intermediate representation encodes, not only the final transition fit.

The held-out next-state MSE was $2.325\times10^{-4}$ without zero-action samples and improved to $2.156\times10^{-4}$ at $\lambda_0=0.10$. It remained between $2.325\times10^{-4}$ and $2.434\times10^{-4}$ at the three higher ratios. Together, these results support the proposed model design within this run: zero-action coverage steers the observable block toward the simulator-defined post-intervention state while preserving low transition error.

\begin{figure}[t]
\centering
\includegraphics[width=0.48\textwidth, keepaspectratio]{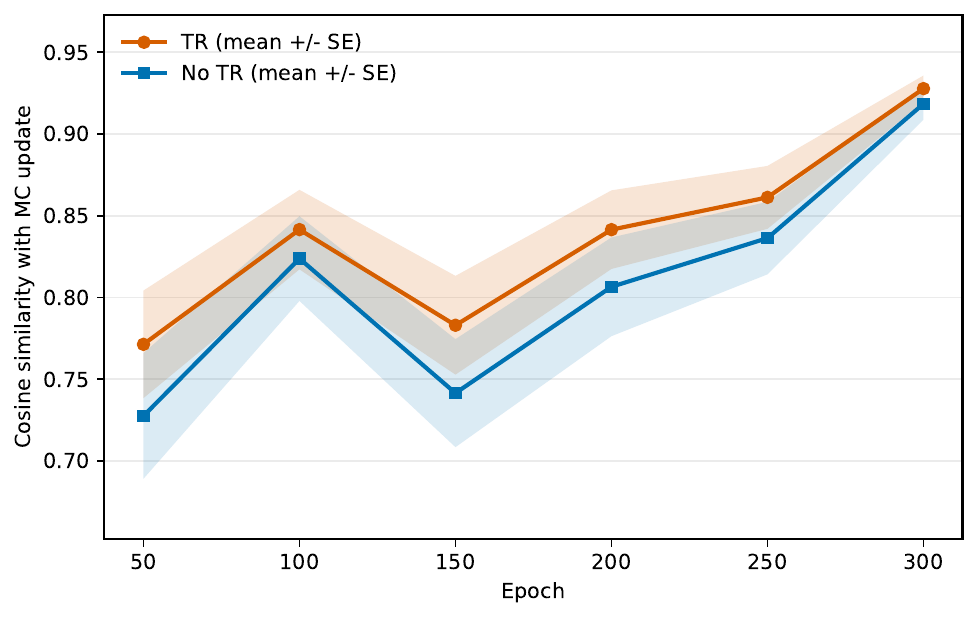}
\caption{Policy-gradient alignment diagnostic on Hopper. We independently train a SAC agent and freeze its actor checkpoints after 50, 100, 150, 200, 250, and 300 training epochs; these are not actors from the main IADD-TR experiments. At each checkpoint, we train paired critics with and without TR from identical initializations on the same 50,000 replay transitions. The critic-induced score-function gradients are evaluated on the same fixed state--initial-action pairs after every critic update and accumulated over 20 critic-training epochs. The finite-sample MC reference uses 256 held-out states, 32 trajectories per state, and a horizon of 500. Curves report the mean over five paired random seeds, and shaded bands indicate mean $\pm$ one standard error of the mean (s.e.m.). Higher cosine similarity indicates closer directional agreement with the MC policy-gradient reference.}
\label{fig:tr_policy_gradient_alignment}
\end{figure}

\textbf{Validation of the Targeted Policy-Gradient Estimator (Q4).}
Q4 directly tests the estimator-specific claim behind TR: the targeted critic should produce a policy-gradient direction closer to the long-horizon return signal used to evaluate the policy. We conduct a paired diagnostic on Hopper using independently trained SAC actors frozen at six checkpoints from 50 to 300 epochs. At each checkpoint, critics with and without TR are trained from identical initializations using the same 50,000 replay transitions, minibatch order, and stochastic target actions. Their score-function gradients are evaluated on identical held-out state--initial-action pairs and accumulated over 20 critic-training epochs, isolating the effect of the targeted correction from actor, data, and optimization variation.

We form a finite-sample MC reference on the same state--initial-action pairs using 32 trajectories per state and a horizon of 500. The MC, TR, and No-TR estimates share the same actions and policy-score vectors and differ only in the return or critic value used to weight those scores. Cosine similarity with this MC reference therefore measures directional accuracy without conflating it with gradient scale.

Fig.~\ref{fig:tr_policy_gradient_alignment} shows that TR achieves higher mean cosine similarity than No-TR at every actor checkpoint. The improvement is largest at epoch 50, where similarity increases from $0.728$ to $0.771$, and remains positive at epoch 300, increasing from $0.918$ to $0.928$. Across the six checkpoints, the paired gain ranges from $0.009$ to $0.044$. The larger gains at earlier and intermediate checkpoints are consistent with TR being most useful when the baseline critic provides a less accurate update direction.

Q4 therefore supports the policy-learning mechanism proposed in the paper: after targeted correction, the critic-induced policy-gradient estimate is more closely aligned with the MC reference. Together with the return improvements in Q2, this result links the estimator-level correction to more reliable actor updates. The MC reference is itself a finite-sample rollout estimate, and cosine similarity evaluates direction rather than magnitude; consequently, this diagnostic does not claim exact gradient recovery or guarantee monotonic return improvement.

\section{Conclusion}
\label{sec:conclusion}

This paper presents IADD-TR, a framework for reducing model-learning and policy-learning bias in model-based reinforcement learning. IADD decomposes transition dynamics into an action-intervention stage and a natural evolution stage without direct action input. A zero-action anchor aligns the observable component of the learned intermediate representation with the physical state space. TR further improves policy learning by explicitly addressing the mismatch between replay-based critic learning and policy-gradient optimization. Across five MuJoCo continuous-control tasks, IADD-TR performed competitively against representative model-free and model-based baselines and learned faster on several tasks. These results suggest that separating intervention effects from natural evolution and improving the alignment between critic learning and policy optimization can effectively mitigate learning bias in model-based reinforcement learning. The applicability of IADD-TR is currently bounded by the identifiability assumptions and the availability of a meaningful zero-action anchor. 

\appendices
\section{Implementation Details}
\label{app:implementation_details}

We summarize the implementation and training settings used in the MuJoCo experiments in Tables~\ref{tab:common_hyperparams} and~\ref{tab:environment_hyperparams}.
The rollout-length schedules and the remaining model-based training settings follow MBPO, while the three additional settings introduced by IADD-TR are fixed across all environments.
For a schedule denoted by $x\rightarrow y$ over epochs $a\rightarrow b$, the value at epoch $e$ is the thresholded linear function
$\min\{\max[x+(e-a)(y-x)/(b-a),x],y\}$.

\begin{table}[h]
\small
\centering
\caption{Common implementation settings and proposed-method hyperparameters for the MuJoCo experiments.}
\label{tab:common_hyperparams}
\begin{tabular}{|>{\centering\arraybackslash}m{0.55\columnwidth}|>{\centering\arraybackslash}m{0.30\columnwidth}|}
\hline
\textbf{Setting} & \textbf{Value} \\
\hline
MuJoCo version & 2.0.0 \\
\hline
Environment steps per epoch & 1000 \\
\hline
Model rollouts per environment step & 400 \\
\hline
Ensemble size & 7 \\
\hline
Model learning rate & $1{\times}10^{-3}$ \\
\hline
Agent learning rate & $3{\times}10^{-4}$ \\
\hline
Zero-action collection ratio & $\lambda_0=0.1$ \\
\hline
Targeted-regularization weight & $\beta=1$ \\
\hline
Random seeds & 5 \\
\hline
\end{tabular}
\end{table}

\begin{table}[h]
\scriptsize
\centering
\caption{Environment-specific hyperparameter settings for the MuJoCo experiments.}
\label{tab:environment_hyperparams}
\setlength{\tabcolsep}{2pt}
\renewcommand{\arraystretch}{1.05}
\begin{tabular}{|>{\centering\arraybackslash}m{0.21\columnwidth}|>{\centering\arraybackslash}m{0.12\columnwidth}|>{\centering\arraybackslash}m{0.15\columnwidth}|>{\centering\arraybackslash}m{0.28\columnwidth}|>{\centering\arraybackslash}m{0.15\columnwidth}|}
\hline
\textbf{Env.} & \textbf{Epochs} &
\shortstack{\textbf{Updates per}\\\textbf{step}} &
\textbf{Model horizon} &
\textbf{Network} \\
\hline
HalfCheetah & 90 & 40 & 1 & \shortstack{MLP\\$4\times200$} \\
\hline
Walker2d & 200 & 20 & 1 & \shortstack{MLP\\$4\times200$} \\
\hline
Ant & 200 & 20 & $1\rightarrow25$ over epochs $20\rightarrow100$ & \shortstack{MLP\\$4\times200$} \\
\hline
Hopper & 60 & 20 & $1\rightarrow15$ over epochs $20\rightarrow100$ & \shortstack{MLP\\$4\times200$} \\
\hline
Humanoid & 200 & 20 & $1\rightarrow25$ over epochs $20\rightarrow300$ & \shortstack{MLP\\$4\times400$} \\
\hline
\end{tabular}
\end{table}

\section{Controlled Synthetic Dynamics Study}
\label{app:synthetic_dynamics_study}

We construct a controlled synthetic dynamics environment using a two-stage ordinary differential equation (ODE) system to evaluate whether zero-action coverage enables recovery of the simulator-defined observable post-intervention state $\tilde{s}_{t,o}$. The state is two-dimensional, $s_t=(x_t,v_t)\in\mathbb{R}^2$, and the action is scalar, $a_t\in[-1,1]$. Each transition is generated by two stages:
\begin{equation}
\label{eq:appendix_synthetic_dynamics_transition}
s_t,a_t \rightarrow \tilde{s}_{t,o} \rightarrow s_{t+1},
\end{equation}
where $\tilde{s}_{t,o}=(\tilde{x}_t,\tilde{v}_t)$ is the simulator-recorded observable post-intervention state. The first stage is an instantaneous intervention on velocity,
\begin{equation}
\label{eq:appendix_synthetic_dynamics_intervention}
\tilde{x}_t=x_t,
\qquad
\tilde{v}_t=v_t+\kappa_a\tanh(a_t),
\end{equation}
with intervention scale $\kappa_a=0.8$. Therefore, the zero-action anchor holds exactly: when $a_t=0$, $\tilde{s}_{t,o}=s_t$.

The second stage is an action-independent natural evolution. Starting from $\tilde{s}_{t,o}$, the system follows a damped oscillator ODE,
\begin{equation}
\label{eq:appendix_synthetic_dynamics_natural}
\frac{dx}{dt_c}=v,
\qquad
\frac{dv}{dt_c}=-\Omega^2x-\kappa_d v,
\end{equation}
where $t_c$ denotes continuous time, $\Omega=1.2$ is the natural frequency, and $\kappa_d=0.25$ is the damping coefficient. We integrate this ODE over $\Delta t=0.15$ using fixed-step fourth-order Runge–Kutta (RK4) with eight substeps, giving the passive flow map $F_0$ and the next state $s_{t+1}=F_0(\tilde{s}_{t,o})$. Initial states are sampled uniformly from $x_t\in[-2,2]$ and $v_t\in[-1.5,1.5]$. Ordinary replay actions are sampled from the state-correlated behavior policy $a_t=\tanh(1.1x_t-0.7v_t+\zeta_t)$, where $\zeta_t\sim\mathcal{N}(0,0.25^2)$, and then clipped to the action range. Consequently, ordinary replay tuples mix intervention effects with passive evolution.

For each zero-action ratio $\lambda_0\in\{0,0.10,0.20,0.30,0.40\}$, we keep the total training size fixed at $N=100{,}000$. The training set contains $\lfloor\lambda_0N\rfloor$ exact zero-action transitions and $N-\lfloor\lambda_0N\rfloor$ ordinary transitions. The observed training next states contain independent Gaussian noise with standard deviation $0.02$. Evaluation uses $M=10{,}000$ clean held-out transitions. All ratios use the same evaluation set and nested ordinary and zero-action data pools, so the mixture ratio is the only changing factor. The reported sweep uses random seed 1, which determines the model initialization and generated data pools. The simulator records $\tilde{s}_{t,o}$, but it is never provided as a training target on nonzero-action samples.

Figure~\ref{fig:synthetic_dynamics_study} and Table~\ref{tab:synthetic_dynamics_results} report this single completed training run. The 10,000 held-out transitions are evaluation observations rather than independent training repetitions, so no cross-seed uncertainty is claimed.

We train a three-member ensemble of two-stage models for 100 epochs. Each stage contains two fully connected layers with hidden width four and a Swish activation. The action-intervention model produces two-dimensional observable and auxiliary latent blocks,
\begin{equation}
p_{\mathrm{act}}(s_t,a_t)
=
\bigl(\widehat{\tilde{s}}_{t,o},\widehat{\tilde{s}}_{t,l}\bigr)
\in\mathbb{R}^{2}\times\mathbb{R}^{2},
\end{equation}
giving a four-dimensional intermediate representation. The natural evolution model receives this complete representation and directly predicts the next state $s_{t+1}$.

All inputs and next-state targets are standardized using the training set. For each ensemble member $k$, the observable intermediate output is defined by
\begin{equation}
\widehat{\tilde{s}}^{(k)}_{i,o}
=
\begin{cases}
s_i, & a_i=0,\\
p^{(k)}_{\mathrm{act},o}(s_i,a_i), & a_i\ne 0.
\end{cases}
\end{equation}
The latent block $\widehat{\tilde{s}}^{(k)}_{i,l}$ remains learned for both zero and nonzero actions. The action-intervention and natural evolution models are jointly optimized using Adam with a learning rate of \(10^{-3}\), a batch size of 256, and the Gaussian negative log-likelihood objective $\widehat{\mathcal{L}}=\widehat{\mathcal{L}}_{\mathrm{NLL}}$. Thus, zero-action samples fix the observable input to the natural evolution stage through the hard model definition, without an auxiliary intermediate-state loss or an anchoring weight.

We evaluate the observable post-intervention recovery error in physical state coordinates,
\begin{equation}
\label{eq:appendix_mid_mse_metric}
\mathrm{MSE}_{\mathrm{mid}}
=
\frac{1}{KMd_o}\sum_{k=1}^{K}\sum_{i=1}^{M}
\left\|
\widehat{\tilde{s}}^{(k)}_{i,o}-\tilde{s}_{i,o}
\right\|_2^2,
\end{equation}
the Pearson correlation between the ensemble-mean recovered action effect $\widehat{\tilde{s}}_{i,o}-s_i$ and the true effect $\tilde{s}_{i,o}-s_i$, and the held-out next-state MSE.

\begin{table}[t]
\centering
\caption{Results of the controlled synthetic dynamics study for random seed 1 with 100,000 training transitions and 10,000 clean held-out transitions per ratio. Lower MSE and higher Pearson correlation are better.}
\label{tab:synthetic_dynamics_results}
\resizebox{\columnwidth}{!}{%
\begin{tabular}{|c|c|c|c|}
\hline
$\lambda_0$ & $\mathrm{MSE}_{\mathrm{mid}}\downarrow$ & \textbf{Effect Pearson}$\uparrow$ & $\mathrm{MSE}_{\mathrm{next}}\downarrow$ \\
\hline
0.00 & $4.0435$ & $0.3455$ & $2.3250\times10^{-4}$ \\
\hline
0.10 & $0.3794$ & $0.7178$ & $2.1564\times10^{-4}$ \\
\hline
0.20 & $0.2378$ & $0.6039$ & $2.3252\times10^{-4}$ \\
\hline
0.30 & $0.03725$ & $0.8876$ & $2.4168\times10^{-4}$ \\
\hline
0.40 & $0.01090$ & $0.9735$ & $2.4344\times10^{-4}$ \\
\hline
\end{tabular}%
}
\end{table}

Table~\ref{tab:synthetic_dynamics_results} shows progressively lower intermediate-state MSE as the zero-action ratio increases. From $\lambda_0=0$ to $\lambda_0=0.40$, the post-intervention-state MSE decreases by $99.7\%$, while the effect correlation increases from $0.3455$ to $0.9735$. Because the model copies the observable intermediate output from the current state at $a=0$, its zero-action error is at numerical precision. This hard-constraint error holds for every ratio and is therefore not reported as a learned metric.

The next-state MSE reaches its minimum of $2.1564\times10^{-4}$ at $\lambda_0=0.10$ and remains below $2.44\times10^{-4}$ at every ratio. Thus, the reported run shows improved observable-state recovery while maintaining low one-step prediction error. Because changing $\lambda_0$ also changes the number of nonzero-action training transitions, the sweep does not isolate the cause of the small differences among nonzero ratios.

\section{Proof of Theorem \ref{thm:point_ident}}
\label{app:proof_pointwise}

\begin{proof}
For each reachable context $\mathbf u=(s,a)$, let $\nu_{\mathbf u}$ and $\bar \nu_{\mathbf u}$ denote the conditional distributions of the complete intermediate representations under the two parameterizations. Observational equivalence gives
\begin{equation}
\label{eq:appendix_equal_transition_pushforwards}
(p_{\mathrm{env}})_{\#}\nu_{\mathbf u}
=
(\bar p_{\mathrm{env}})_{\#}\bar \nu_{\mathbf u}
\end{equation}
for every $\mathbf u$, where $(\cdot)_{\#}$ denotes the pushforward of a probability distribution.

By A1, both natural evolution mechanisms have inverses on the common reachable next-state support. Hence they induce the context-independent change of intermediate coordinates
\begin{equation}
\label{eq:appendix_environment_coordinate_change}
H
:=
\bar p_{\mathrm{env}}^{-1}\circ p_{\mathrm{env}}.
\end{equation}
Applying $\bar p_{\mathrm{env}}^{-1}$ to Eq.~\eqref{eq:appendix_equal_transition_pushforwards} yields
\begin{equation}
\label{eq:appendix_intermediate_pushforward}
H_{\#}\nu_{\mathbf u}=\bar \nu_{\mathbf u}.
\end{equation}

Let $\pi_o$ denote projection onto the observable coordinate. Take any reachable intermediate representation $\mathbf z$ whose image $y=p_{\mathrm{env}}(\mathbf z)$ lies in the conditional next-state support of some zero-action context $(s,a_0)$. This is precisely the zero-action-covered portion of the reachable support specified in the theorem. Observational equivalence makes this conditional support the same under both parameterizations. Because the natural evolution mechanisms are one-to-one and have continuous inverses on their reachable ranges, $\mathbf z$ belongs to the intermediate support under $(s,a_0)$, while $H(\mathbf z)$ belongs to the corresponding barred support. The defining zero-action identities of the two parameterizations then give $\pi_o(\mathbf z)=s$ and $\pi_o(H(\mathbf z))=s$. Therefore,
\begin{equation}
\label{eq:appendix_observable_coordinate_fixed}
\pi_o(H(\mathbf z))=\pi_o(\mathbf z).
\end{equation}

For an arbitrary reachable $\mathbf u=(s_t,a_t)$ whose conditional next-state support is covered by zero-action contexts, $\nu_{\mathbf u}$ is concentrated on intermediate representations whose observable coordinate is the deterministic value $p_{\mathrm{act},o}(s_t,a_t)$. Similarly, $\bar \nu_{\mathbf u}$ has the deterministic observable coordinate $\bar p_{\mathrm{act},o}(s_t,a_t)$. Combining Eqs.~\eqref{eq:appendix_intermediate_pushforward} and~\eqref{eq:appendix_observable_coordinate_fixed} shows that these two point masses coincide. Thus
\[
\bar p_{\mathrm{act},o}(s_t,a_t)
=
p_{\mathrm{act},o}(s_t,a_t),
\]
which proves pointwise identifiability on the zero-action-covered reachable support.
\end{proof}

\section{Proof of Theorem \ref{thm:latent_component_ident}}
\label{app:proof_latent}

Fix a reachable observable post-intervention state $s'_o$ and define its state-action context fiber by
\begin{equation}
\label{eq:observable_context_fiber}
\mathcal U(s'_o)
:=
\{\mathbf u=(s,a):p_{\mathrm{act},o}(s,a)=s'_o\}.
\end{equation}
Because both parameterizations belong to the factorized model class in Eq.~\eqref{eq:act_factorization}, their latent conditional densities already satisfy
\[
\begin{aligned}
p_l(\tilde{s}_{t,l}\mid\mathbf u)
&=
\prod_{j=1}^{d_l}p_{l,j}(\tilde{s}_{t,l,j}\mid\mathbf u),
\\
\bar p_l(\bar{s}_{t,l}\mid\mathbf u)
&=
\prod_{j=1}^{d_l}\bar p_{l,j}(\bar{s}_{t,l,j}\mid\mathbf u).
\end{aligned}
\]
Thus, factorization is part of the model definition rather than an additional identifiability assumption. The additional conditions used in the theorem are:
\begin{itemize}
    \item B1 (\underline{Regularity of conditional latent densities}):
    The component densities are positive and twice continuously differentiable on connected latent supports that locally contain product neighborhoods.
    \item B2 (\underline{Sufficient variability within an observable fiber}):
    Fix $\mathbf u_0\in\mathcal U(s'_o)$ and define
    \[
    \bar c_j(x,\mathbf u)
    :=
    \log\bar p_{l,j}(x\mid\mathbf u)
    -
    \log\bar p_{l,j}(x\mid\mathbf u_0).
    \]
    For every reachable $\bar{s}_{t,l}$, there exist $2d_l$ contexts $\mathbf u_1,\dots,\mathbf u_{2d_l}\in\mathcal U(s'_o)$ such that the matrix with columns
    \[
    \mathbf v(\bar{s}_{t,l},\mathbf u_k)
    =
    \left[
    \begin{array}{c}
    \partial_x^2\bar c_1(\bar{s}_{t,l,1},\mathbf u_k)\\
    \vdots\\
    \partial_x^2\bar c_{d_l}(\bar{s}_{t,l,d_l},\mathbf u_k)\\
    \partial_x\bar c_1(\bar{s}_{t,l,1},\mathbf u_k)\\
    \vdots\\
    \partial_x\bar c_{d_l}(\bar{s}_{t,l,d_l},\mathbf u_k)
    \end{array}
    \right]
    \in\mathbb R^{2d_l}
    \]
    has full rank.
\end{itemize}

\begin{proof}
Fix a reachable observable post-intervention state $s'_o$. By A1, the support-level coordinate change $H=\bar p_{\mathrm{env}}^{-1}\circ p_{\mathrm{env}}$ and its inverse are twice continuously differentiable. Theorem~\ref{thm:point_ident} gives $\pi_o\circ H=\pi_o$ on the reachable support. Consequently, $H$ maps each latent fiber bijectively onto the corresponding barred fiber and has the form
\[
H(s'_o,\tilde{s}_{t,l})
=
\bigl(s'_o,H_{l,s'_o}(\tilde{s}_{t,l})\bigr).
\]
Both $H_{l,s'_o}$ and its inverse are twice continuously differentiable. Differentiating their composition gives
\[
J_{H_{l,s'_o}^{-1}}\!\left(H_{l,s'_o}(\tilde{s}_{t,l})\right)
J_{H_{l,s'_o}}(\tilde{s}_{t,l})
=
I_{d_l},
\]
so $J_{H_{l,s'_o}}$ is nonsingular. Thus the support inverses in A1 induce the latent-coordinate diffeomorphism
\[
\bar{s}_{t,l}
=
H_{l,s'_o}(\tilde{s}_{t,l}).
\]
For notational simplicity within this fixed observable fiber, write $H=H_{l,s'_o}$ and $H=(H_1,\dots,H_{d_l})$.

For every $\mathbf u\in\mathcal U(s'_o)$, let $P_l^{\mathbf u}$ and $\bar P_l^{\mathbf u}$ denote the two conditional latent distributions. Because the observable coordinate is fixed at $s'_o$, restricting Eq.~\eqref{eq:appendix_intermediate_pushforward} to this fiber gives
\[
H_{\#}P_l^{\mathbf u}
=
\bar P_l^{\mathbf u}.
\]
The change-of-variables formula therefore yields
\begin{equation}
\label{eq:appendix_latent_density_change}
p_l(\tilde{s}_{t,l}\mid\mathbf u)
=
\bar p_l(H(\tilde{s}_{t,l})\mid\mathbf u)
\left|\det J_H(\tilde{s}_{t,l})\right|.
\end{equation}
Let
\[
c_j(x,\mathbf u)
:=
\log p_{l,j}(x\mid\mathbf u)
-
\log p_{l,j}(x\mid\mathbf u_0),
\]
and define $\bar c_j$ as in B2. Evaluating the logarithm of Eq.~\eqref{eq:appendix_latent_density_change} at $\mathbf u$ and $\mathbf u_0$ and subtracting cancels the Jacobian term. The component factorizations in Eq.~\eqref{eq:act_factorization} then yield
\begin{equation}
\label{eq:latent_contrast_identity}
\sum_{j=1}^{d_l}c_j(\tilde{s}_{t,l,j},\mathbf u)
=
\sum_{m=1}^{d_l}\bar c_m(H_m(\tilde{s}_{t,l}),\mathbf u).
\end{equation}

Fix two distinct latent coordinate indices $r,k\in\{1,\dots,d_l\}$. Applying $\partial^2/(\partial\tilde{s}_{t,l,r}\partial\tilde{s}_{t,l,k})$ to Eq.~\eqref{eq:latent_contrast_identity} gives
\begin{equation}
\label{eq:latent_mixed_derivative_constraint}
\begin{aligned}
0
=
\sum_{m=1}^{d_l}
\Bigl[
&\partial_x^2\bar c_m(H_m(\tilde{s}_{t,l}),\mathbf u)
\partial_{\tilde{s}_{t,l,r}}H_m(\tilde{s}_{t,l})
\partial_{\tilde{s}_{t,l,k}}H_m(\tilde{s}_{t,l})
\\
&+
\partial_x\bar c_m(H_m(\tilde{s}_{t,l}),\mathbf u)
\partial^2_{\tilde{s}_{t,l,r}\tilde{s}_{t,l,k}}H_m(\tilde{s}_{t,l})
\Bigr].
\end{aligned}
\end{equation}
For fixed $\tilde{s}_{t,l},r,k$, the derivatives of $H$ do not depend on the context $\mathbf u$. Evaluating Eq.~\eqref{eq:latent_mixed_derivative_constraint} at the $2d_l$ contexts supplied by B2 gives a full-rank homogeneous linear system. Its unique solution is
\begin{equation}
\label{eq:latent_no_coordinate_mixing}
\begin{aligned}
\partial_{\tilde{s}_{t,l,r}}H_m(\tilde{s}_{t,l})
\partial_{\tilde{s}_{t,l,k}}H_m(\tilde{s}_{t,l})
&=0,
\\
\partial^2_{\tilde{s}_{t,l,r}\tilde{s}_{t,l,k}}H_m(\tilde{s}_{t,l})
&=0.
\end{aligned}
\end{equation}
for every $m$ and every pair $r\neq k$.

Equation~\eqref{eq:latent_no_coordinate_mixing} shows that each output coordinate $H_m$ can depend locally on at most one input coordinate. Because $H$ is a diffeomorphism, its Jacobian is nonsingular; consequently, every row and every column of $J_H$ contains exactly one nonzero entry. This assignment defines a permutation $\sigma_{l,s'_o}$. The connected-product-support condition in B1 prevents the assignment from switching within the latent support, so integration along each assigned coordinate yields a one-dimensional diffeomorphism $h_{m,s'_o}$ satisfying
\[
H_m(\tilde{s}_{t,l})
=
h_{m,s'_o}\!\left(\tilde{s}_{t,l,\sigma_{l,s'_o}(m)}\right).
\]
This proves Eq.~\eqref{eq:conditional_latent_component_ident}.

\end{proof}

\section{Proof of Theorem \ref{thm:eif_pg}}
\label{app:proof_eif}

\begin{proof}
Let $O=(S,A,Y^{\pi_\theta})$ denote the augmented replay sample introduced before Theorem~\ref{thm:eif_pg}, with $S\sim\mathcal{D}_S$ and $A\mid S\sim\pi_\beta(\cdot\mid S)$. Bellman consistency gives
\begin{equation}
\mathbb{E}_{\mathcal D}
\left[Y^{\pi_\theta}\mid S=s,A=a\right]
=Q^{\pi_\theta}(s,a).
\end{equation}

Consider an arbitrary regular parametric submodel $\{\mathcal D_\varepsilon\}$ through the replay distribution $\mathcal D$, with score $\ell(O)$. Along this path, define
\begin{equation}
Q_\varepsilon(s,a)
:=
\mathbb{E}_{\mathcal D_\varepsilon}
\left[Y^{\pi_\theta}\mid S=s,A=a\right]
\end{equation}
and
\begin{equation}
\psi_\varepsilon
=
\mathbb{E}_{S\sim\mathcal D_{\varepsilon,S}}
\left[
\mathbb{E}_{\tilde a\sim\pi_\theta(\cdot\mid S)}
\left[
Q_\varepsilon(S,\tilde a)g_\theta(S,\tilde a)
\right]
\right].
\end{equation}
At $\varepsilon=0$, $Q_0=Q^{\pi_\theta}$ and $\psi_0=\psi_{\mathrm{pg}}^{\mathcal D}$. We hold $\pi_\theta$ and $g_\theta$ fixed throughout the pathwise derivative.

The derivative of the conditional Bellman mean is
\begin{equation}
\begin{aligned}
\left.
\frac{d}{d\varepsilon}Q_\varepsilon(s,a)
\right|_{\varepsilon=0}
&=
\mathbb{E}_{\mathcal D}
\left[
\bigl(Y^{\pi_\theta}-Q^{\pi_\theta}(s,a)\bigr)
\right.\\
&\qquad\left.
\times\ell(O)\mid S=s,A=a
\right].
\end{aligned}
\end{equation}
After changing the action measure from $\pi_\beta$ to $\pi_\theta$, its contribution to the derivative of $\psi_\varepsilon$ becomes
\begin{equation}
\mathbb{E}_{\mathcal D}
\left[
\frac{\pi_\theta(A\mid S)}{\pi_\beta(A\mid S)}
\bigl(Y^{\pi_\theta}-Q^{\pi_\theta}(S,A)\bigr)
g_\theta(S,A)\ell(O)
\right].
\end{equation}

The state-marginal contribution is
\begin{equation}
\mathbb{E}_{\mathcal D}
\left[
\left(
\mathbb{E}_{\tilde a\sim\pi_\theta(\cdot\mid S)}
\left[
Q^{\pi_\theta}(S,\tilde a)g_\theta(S,\tilde a)
\right]
-\psi_{\mathrm{pg}}^{\mathcal D}
\right)
\ell(O)
\right].
\end{equation}
Perturbing the behavior-policy density $\pi_\beta(a\mid s)$ contributes no additional term because the target functional integrates actions under the fixed policy $\pi_\theta$. Combining the two contributions gives
\begin{equation}
\left.
\frac{d}{d\varepsilon}\psi_\varepsilon
\right|_{\varepsilon=0}
=
\mathbb{E}_{\mathcal D}
\left[
\varphi_{\mathrm{pg}}^{\mathcal D}(O)\ell(O)
\right],
\end{equation}
where $\varphi_{\mathrm{pg}}^{\mathcal D}$ is exactly Eq.~\eqref{eq:eif_pg}. Its expectation under $\mathcal D$ is zero because the residual has conditional mean zero and the plug-in term is centered by $\psi_{\mathrm{pg}}^{\mathcal D}$. Thus it is an influence function. Since the replay distribution $\mathcal D$ is modeled nonparametrically, its tangent space is $L_0^2(\mathcal D)$, so the influence function is efficient.
\end{proof}

\section{Proof of Theorem \ref{thm:tr_asymptotic_targeting}}
\label{app:proof_tr}

We analyze the targeted correction action by action, following the functional targeted-regularization construction for continuous treatments~\cite{nie2021vcnet}. Let $O=(S,A,Y^{\pi_\theta})$ denote a replay observation, with $S\sim\mathcal D_S$, $A\mid S\sim e_0(\cdot\mid S)$, and
\begin{equation}
\label{eq:appendix_conditional_q}
m_0(s,a)
:=
\mathbb E\!\left[Y^{\pi_\theta}\mid S=s,A=a\right]
=
Q^{\pi_\theta}(s,a).
\end{equation}
The action-wise target is the marginal conditional-mean curve
\begin{equation}
\label{eq:appendix_actionwise_mean}
\mu_0(a)
:=
\mathbb E_{S\sim\mathcal D_S}
\left[m_0(S,a)\right].
\end{equation}
This target is the mean action value at action $a$; no contrast with a reference action is taken.

Let $\widehat m(s,a)=Q_\omega(s,a)$ be a preliminary critic and let $\widehat e(a\mid s)$ be a candidate replay action-density estimate. For an action-wise correction $\epsilon(a)$, define the adjusted conditional mean
\begin{equation}
\label{eq:appendix_vcnet_adjustment}
m_{\epsilon}(s,a)
=
\widehat m(s,a)
+
\frac{\epsilon(a)}{\widehat e(a\mid s)}.
\end{equation}
This is the population counterpart of the adjusted critic in Eq.~\eqref{eq:adjusted_critic}. Its population targeted risk is
\begin{equation}
\label{eq:appendix_residual_risk}
\mathcal R_\infty(\epsilon)
=
\mathbb E
\left[
\left\{
Y^{\pi_\theta}
-\widehat m(S,A)
-\frac{\epsilon(A)}{\widehat e(A\mid S)}
\right\}^2
\right].
\end{equation}

The argument uses the following conditions. First, $e_0$ and $\widehat e$ are positive and bounded away from zero on the relevant compact action support. Second, the conditional Bellman noise has mean zero and finite second moment, and the critic, inverse action densities, and correction functions are square integrable. Third, the action-wise correction class can approximate the population minimizer below, and the empirical targeted risk converges uniformly to Eq.~\eqref{eq:appendix_residual_risk}. The fitted correction is an $o_P(1)$ approximate empirical-risk minimizer. For continuous actions, these conditions are imposed in $L^2(\mathcal D_A)$ over the replay action marginal, as in functional targeted regularization. Finally, either the preliminary critic is consistent for $m_0$, or $\widehat e$ is consistent for $e_0$. When a practical bootstrapped Bellman target is used in place of $Y^{\pi_\theta}$, its contribution to the targeted risk is additionally assumed to be $o_P(1)$.

\begin{proof}
Because the correction is indexed by the observed action, the first-order condition for Eq.~\eqref{eq:appendix_residual_risk} can be evaluated action by action. For every action value $a$ in the relevant support, the population minimizer satisfies
\begin{equation}
\mathbb E\left[
\left.
\frac{
Y^{\pi_\theta}-\widehat m(S,A)
-\epsilon^*(A)/\widehat e(A\mid S)
}{
\widehat e(A\mid S)
}
\right|A=a
\right]
=0.
\end{equation}
Consequently,
\begin{equation}
\label{eq:oracle_epsilon_tr}
\epsilon^*(a)
=
\frac{
\mathbb E\left[
\left.
\{Y^{\pi_\theta}-\widehat m(S,A)\}/\widehat e(A\mid S)
\right|A=a
\right]
}{
\mathbb E\left[
\left.
\widehat e(A\mid S)^{-2}
\right|A=a
\right]
}.
\end{equation}
Uniform convergence, sufficient approximation capacity, and approximate empirical-risk minimization give
\begin{equation}
\label{eq:epsilon_convergence_tr}
\left\|
\epsilon_\xi-\epsilon^*
\right\|_{L^2(\mathcal D_A)}
=o_P(1).
\end{equation}

It remains to show that the adjusted conditional-mean curve is doubly robust. Let $p_A(a)$ denote the replay action marginal density. Since
\[
p_A(a)
=
\mathbb E_{S\sim\mathcal D_S}
\left[e_0(a\mid S)\right],
\]
conditioning on $A=a$ in Eq.~\eqref{eq:oracle_epsilon_tr} yields
\begin{equation}
\label{eq:appendix_actionwise_epsilon}
\epsilon^*(a)
=
\frac{
\mathbb E_S\left[
\dfrac{e_0(a\mid S)}{\widehat e(a\mid S)}
\{m_0(S,a)-\widehat m(S,a)\}
\right]
}{
\mathbb E_S\left[
\dfrac{e_0(a\mid S)}{\widehat e(a\mid S)^2}
\right]
}.
\end{equation}
Define the population curve induced by this correction as
\begin{equation}
\label{eq:appendix_adjusted_mean_curve}
\mu^*(a)
=
\mathbb E_S
\left[
\widehat m(S,a)
+
\frac{\epsilon^*(a)}{\widehat e(a\mid S)}
\right].
\end{equation}
Writing $r_a(S)=m_0(S,a)-\widehat m(S,a)$ and substituting Eq.~\eqref{eq:appendix_actionwise_epsilon} into Eq.~\eqref{eq:appendix_adjusted_mean_curve} gives
\begin{equation}
\label{eq:targeted_remainder_identity}
\begin{aligned}
\mu^*(a)-\mu_0(a)
=&
-\mathbb E_S[r_a(S)]
\\
&+
\frac{
\mathbb E_S[\widehat e(a\mid S)^{-1}]
\mathbb E_S\left[
\dfrac{e_0(a\mid S)}{\widehat e(a\mid S)}r_a(S)
\right]
}{
\mathbb E_S\left[
\dfrac{e_0(a\mid S)}{\widehat e(a\mid S)^2}
\right]
}.
\end{aligned}
\end{equation}

If the preliminary critic is correct, then $r_a(S)=0$, and both terms on the right-hand side of Eq.~\eqref{eq:targeted_remainder_identity} vanish. If the replay action density is correct, then $\widehat e=e_0$, and Eq.~\eqref{eq:appendix_actionwise_epsilon} reduces to
\begin{equation}
\epsilon^*(a)
=
\frac{
\mathbb E_S[r_a(S)]
}{
\mathbb E_S[e_0(a\mid S)^{-1}]
}.
\end{equation}
The correction term in Eq.~\eqref{eq:appendix_adjusted_mean_curve} then equals $\mathbb E_S[r_a(S)]$, so $\mu^*(a)=\mu_0(a)$ even when the preliminary critic is misspecified. Thus the population adjusted curve is correct if either the critic or the replay action density is correct.

For the fitted correction, define
\begin{equation}
\label{eq:appendix_empirical_mean_curve}
\widehat\mu^{\mathrm{TR}}(a)
=
\frac{1}{n}
\sum_{i=1}^{n}
\left[
Q_\omega(s_i,a)
+
\frac{\epsilon_\xi(a)}{\widehat e(a\mid s_i)}
\right].
\end{equation}
Equation~\eqref{eq:epsilon_convergence_tr}, overlap, and uniform convergence of the state averages imply that $\widehat\mu^{\mathrm{TR}}$ converges to $\mu^*$ in $L^2(\mathcal D_A)$. Combining this convergence with Eq.~\eqref{eq:targeted_remainder_identity} gives
\begin{equation}
\left\|
\widehat\mu^{\mathrm{TR}}-\mu_0
\right\|_{L^2(\mathcal D_A)}
=o_P(1)
\end{equation}
whenever either $\widehat m$ is consistent for $m_0$ or $\widehat e$ is consistent for $e_0$. This proves double robustness of the action-wise marginal conditional-mean curve.
\end{proof}

 \bibliographystyle{IEEEtran}
 \bibliography{tpami}

\end{document}